%% file: paper.tex
\documentclass{article}

\usepackage{microtype}
\usepackage{graphicx}
\usepackage{subfigure}
\usepackage{booktabs} 
\usepackage{enumerate}
\usepackage{mathextra}
\usepackage[algo2e]{algorithm2e}
\usepackage{amssymb}
\usepackage{hyperref}

\usepackage[accepted]{icml2021}

\newcommand{\mytitle}{Bias-Robust Bayesian Optimization via Dueling Bandits}
\icmltitlerunning{\mytitle}

\newcommand{\ip}[1]{\langle#1\rangle}

\usepackage{tcolorbox}
\usepackage[textsize=tiny, disable]{todonotes}

\newtcbox{\entoure}[1][red]{on line,
	arc=3pt,colback=#1!10!white,colframe=#1!50!black,
	before upper={\rule[-3pt]{0pt}{10pt}},boxrule=1pt,
	boxsep=0pt,left=2pt,right=2pt,top=1pt,bottom=-1pt}

\usepackage{tcolorbox}

\newcommand{\Uniform}{\text{Uniform}}
\newcommand{\Bernoulli}{\text{Bernoulli}}
\newcommand{\xx}[1]{#1^{1},#1^{2}}
\newcommand{\xxt}{\xx{x_t}}
\newcommand{\xxx}{\xx{x}}
\newcommand{\dd}{d}
\newcommand{\deq}{\triangleq}
\newcommand{\fopt}{f}
\newcommand{\hx}{\hat x^*}

\begin{document}

\twocolumn[
\icmltitle{\mytitle}



\icmlsetsymbol{equal}{*}

\begin{icmlauthorlist}
\icmlauthor{Johannes Kirschner}{eth}
\icmlauthor{Andreas Krause}{eth}
\end{icmlauthorlist}

\icmlaffiliation{eth}{Department of Computer Science, ETH Zurich}

\icmlcorrespondingauthor{Johannes Kirschner}{jkirschner@inf.ethz.ch}

\icmlkeywords{Machine Learning, ICML}

\vskip 0.3in
]



\printAffiliationsAndNotice{}  

\begin{abstract}
We consider Bayesian optimization in settings where observations can be {\em adversarially biased}, for example by an uncontrolled hidden confounder. Our first contribution is a reduction of the confounded setting to the {\em dueling bandit} model. Then we propose a novel approach for dueling bandits based on {\em information-directed sampling (IDS)}. Thereby, we obtain the first efficient kernelized algorithm for dueling bandits that comes with cumulative regret guarantees. Our analysis further generalizes a previously proposed semi-parametric linear bandit model to non-linear reward functions, and uncovers interesting links to doubly-robust estimation.
\end{abstract}
\input{parts/introduction.tex}

\input{parts/setting.tex}

\input{parts/reduction.tex}

\input{parts/ids.tex}

\input{parts/experiments.tex}

\input{parts/conclusion.tex}

\section*{Acknowledgements}
This research has received funding from the European Research Council (ERC) under the European Union’s Horizon 2020 research and innovation programme grant agreement No 815943.

\bibliography{references}
\bibliographystyle{icml2021}


\end{document}

%% file: parts/introduction.tex
\section{Introduction}

Bayesian optimization \citep{Mockus1982} is a model-based approach for zero-order global optimization with noisy feedback. It has been successfully applied to many applications such as hyper-parameter tuning of machine learning models, robotics and chemical design. Some variants such as Expected-Improvement \citep{bull2011convergence} or the GP-UCB algorithm \cite{Srinivas2009} come with theoretical guarantees, ensuring convergence to the global optimum in finite time under suitable regularity assumptions. Closely related is the field of bandit algorithms \citep{lattimore2018bandit}, in particular the linear bandit model \cite{abe1999associative,dani2008stochastic,Abbasi2011improved}. 

Most linear bandit algorithms and Bayesian optimization approaches require that the true function is realized in a known reproducing kernel Hilbert space (RKHS) and rely on {\em unbiased} evaluations of the objective. The regularity assumptions raise questions of robustness to miss-specification and adversarial attacks, and addressing these limitations has been the content of several recent works \cite{lykouris2018stochastic,bogunovic2020stochastic,bogunovic2020corruption}. 

We study a setting where the learner's objective is to maximize an unknown function $\fopt: \xX \rightarrow \RR$ with {\em additive confounded feedback},
\begin{align}
	y_t = \fopt(x_t) + b_t + \epsilon_t\,, \label{eq:feedback-intro}
\end{align}
\looseness -1 where $x_t$ is the evaluation point chosen by the learner at time $t$, $\epsilon_t$ is $\sigma^2$-sub-Gaussian (zero-mean) observation noise, and $b_t$ is an additive confounding term. We assume that $b_t$ is chosen by an adversary, but does \emph{not} depend on the input $x_t$. The bias term allows to model the influence of an unobserved and uncontrolled covariate, or a perturbation of the feedback signal imposed by an adversary. One can also interpret \eqref{eq:feedback-intro} as a \emph{contextual} model, where $b_t$ captures the effect of a changing (unobserved) context on the reward. We discuss further examples and applications in Section \ref{ss:examples} below.

\looseness -1 The proposed feedback model \eqref{eq:feedback-intro} generalizes the semi-parametric contextual bandit model studied by \citet{krishnamurthy2018semiparametric}, where $\fopt(x_t) = \ip{x_t, \theta}$ is a linear function defined by a parameter $\theta \in \RR^d$. They show that a doubly-robust least-squares estimator allows to recover reward differences $\ip{x - x', \theta}$ for inputs $x, x' \in \RR^d$ despite the confounding. They further propose an elimination-style algorithm, \emph{bandit orthogonalized semiparametric estimation} (BOSE), which is based on sampling actions from a distribution that minimizes the variance of the estimator. However, finding low variance distributions requires solving a convex-quadratic feasibility problem, which is computationally demanding and in  general leads to sampling distributions with support that spans $\RR^d$.

\paragraph{Contributions}  Our first contribution is two reductions of the confounded feedback model \eqref{eq:feedback-intro} to the {\em dueling bandit} setting. This allows us to leverage existing algorithms for dueling bandits in the confounded observation setting. We then propose the first efficient algorithm for kernelized dueling bandits that comes with theoretical guarantees on the cumulative regret. The approach is based on {\em information-directed sampling (IDS)}, which was recently studied in the context of linear partial monitoring by \citet{kirschner20partialmonitoring}. In particular, we propose an efficient approximation of IDS, that reduces the computation complexity from $\oO(|\xX|^4)$ to $\oO(|\xX|)$ on finite action sets $\xX$. For continuous action sets, the proposed algorithm requires to optimize a (non-convex) acquisition function over the input space, akin to standard Bayesian optimization. 

\subsection{Motivating Examples}\label{ss:examples}

We start by motivating our problem through applications.

\paragraph{Range-Adjusting Measurement Devices} In many real-world optimization tasks, the observed feedback arises from a physical sensing device. Such measurement devices can be subject to calibration errors or might automatically adjust the output range for better sensitivity. For example, in optimization of free electron lasers \cite{kirschner2019adaptive,duris2020bayesian}, the target signal is measured with a gas-detector, which exploits a physical law to amplify the signal. An input voltage is used to control the amplification factor, which requires re-adjustment\footnote{Translating the range-adjusted signal into an absolute value is possible, but not straightforward and comes with other limitations.} with increasing target signal such that the physical relationship between the target and the measured output stays approximately linear\footnote{More precisely, the relationship depends on the photon energy, pulse intensity, and physical properties of the gas involved.} \citep{sorokin2019x, juranic2018swissfel}. Our feedback model allows optimization that is robust to absolute changes in the target signal occurring at any time.

\paragraph{Distributed Optimization of Additive Functions} In high-dimensional settings, previous work on Bayesian optimization often uses structural assumptions to reduce the dependence of the sample complexity on the dimension. One popular choice is additivity \citep{kandasamy2015high}, for example coordinate-wise $\fopt(x) = \sum_{i \in [d]} f_i(x^i)$. To optimize additive functions, we can apply $d$ individual learners to optimize each 1-dimensional component $f_i$ separately. Note that the learners cannot directly evaluate $f_i(x_i)$, but only obtain the global noisy feedback $f(x) + \epsilon$, which depends on the choices $x = (x^1,\dots, x^d)$ of \emph{all} learners. When the learners act in parallel and there is \emph{no communication} possible, the feedback of each learner is confounded as in \eqref{eq:feedback-intro} by the other learner's choices. Our robust approach guarantees that the learners are able to optimize each component successfully, despite the confounding.

\paragraph{Adversarial Attacks} Robustness to adversarial attacks was studied recently in the context of bandit algorithms \citep{lykouris2018stochastic,bogunovic2020stochastic} and Bayesian optimization \cite{bogunovic2020corruption}. In all previous work that we are aware of, the corruption of the feedback is allowed to depend on the actions, with varying assumptions of whether or not the adversary observes the action choice of the learner. Note that our feedback model is more stringent, as it does not allow for action dependence. On the other hand, our model allows for sublinear regret even with a constant corruption in \emph{every} round, whereas in previous work, the regret scales with the total amount of corruption. \looseness=-1

\subsection{Related Work}

\looseness -1 There is a vast amount of literature on bandit algorithms \cite{lattimore2018bandit} and Bayesian optimization \cite{Mockus1982,Srinivas2009,Shahriari2016,frazier2018tutorial}. Our feedback model is a generalization of the semi-parametric linear bandit setting proposed by \citet{krishnamurthy2018semiparametric} with applications for example in mobile health \cite{tewari2017ads}. A variant of Thompson sampling was analyzed in the same setting by \citet{kim2019contextual}, which they show outperforms the BOSE approach by \citet{krishnamurthy2018semiparametric}, but the frequentist regret bound they derive has an extra factor in the dimension. The Thompson sampling variant is also computationally more efficient, but requires to explicitly compute the probabilities that each action is optimal under the posterior distribution. For Bayesian optimization, robust variants have been considered recently, for example with adversarial perturbations of the input \citep{bogunovic2018adversarially}, corruption of the output \citep{bogunovic2020corruption} and distributionally robust optimization \citep{kirschner2020distributionally}. In the context of adversarial attacks, there is an increasing body of work \cite{lykouris2018stochastic,li2019stochastic,liu2020action,gupta2019better,bogunovic2020corruption,bogunovic2020stochastic}, however the feedback model differs from ours, see also the discussion in Section \ref{ss:examples}.

Of particular relevance to our work is the (stochastic) dueling bandit setting \cite{yue2012k,sui2018advancements,bengs2021preference} and kernelized variants \cite{sui2017correlational,sui2017multi}. Early work by \citet{yue2009interactively} applied the dueling bandit model to the optimization setting with continuous action sets and concave reward functions, and established a connection to gradient-based optimization. However, to the best of our knowledge, none of the previous works provide bounds on the cumulative regret in the kernelized (non-concave) setting. A kernelized algorithm with theoretical guarantees that requires point evaluations \emph{and} dueling feedback is by \citet{xu2020zeroth}. Closely related is also the work by \citet{saha2020regret} on linear dueling bandits with possibly infinite input spaces. They establish a connection to the generalized linear bandit model (GLM) and propose an algorithm which, similar to ours, relies on finding an informative action pair with low regret. However, gap-dependent bounds and a kernelized variant was not provided, and for finite action sets $\xX$, their algorithm requires $\oO(|\xX|^2)$ computation steps per round. Recent work by \citet{agarwal2021stochastic} considers the finite-armed dueling bandit setting with adversarial corruptions of the feedback.
Lastly, we remark that most previous work on dueling bandits considers binary feedback, whereas here we are interested in quantitative feedback on the reward-difference between the chosen action pair. Formally, we consider sub-Gaussian dueling feedback which includes the Bernoulli likelihood, but does not exploit the heteroscedasticity of binary observations.

%% file: parts/setting.tex
\section{Setting}

Let $\xX \subset \RR^d$ be a compact input space and $\fopt: \xX \rightarrow \RR$ a fixed and unknown objective function. In each round $t=1, \dots, n$, the learner chooses an action $x_t \in \xX$ and observes the confounded outcome
\begin{align*} 
	y_t = \fopt(x_t) + b_t + \epsilon_t\,,
\end{align*}
where $\epsilon_t$ is $\sigma^2$-sub-Gaussian, conditionally independent noise, and $b_t \in \RR^d$ is an unobserved and possibly time-dependent confounding term. We assume that $b_t$ does not depend on the current action $x_t$ chosen by the learner, and satisfies one of the following assumptions:
\begin{enumerate}[a)]
	\item The bias $b_t$ is bounded, $|b_t| \leq C_{max}$ and fixed at the beginning of round $t$, but can otherwise arbitrarily depend on $(x_s, y_s)_{s=1}^{t-1}$.
	\item The difference between two consecutive bias terms is bounded, $|b_t - b_{t-1}| \leq D_{max}$ and $b_t$ is fixed at the beginning of round $t-1$, but can otherwise arbitrarily depend on $(x_s, y_s)_{s=1}^{t-2}$.
\end{enumerate}
Which assumption is used is specified in the relevant context.
Let $x^* \in \argmax_{x \in \xX} \fopt(x)$ be the optimal action. 
The suboptimality gap is $\Delta(x) = \fopt(x^*) - \fopt(x)$. The learner's objective is to maximize  the cumulative reward $\sum_{t=1}^n \fopt(x_t)$, or equivalently minimize the {\em regret},
\begin{align*}
	R_n = \sum_{t=1}^n \fopt(x^*) - \fopt(x_t) = \sum_{t=1}^n \Delta(x_t)\,.
\end{align*}
\looseness=-1 For the analysis, we assume that the function $\fopt$ is in a known \emph{reproducing kernel Hilbert space} (RKHS) $\hH$ with associated kernel $k : \xX \times \xX \rightarrow \RR$ and bounded Hilbert norm $\|f\|_{\hH} \leq B$. This is a standard assumption in Bayesian optimization \citep{Srinivas2009,Chowdhury2017} and justifies the use of kernelized least-squares regression, formally introduced in Section \ref{sec:ids}. We further require that the kernel function is bounded, $k(x,x) \leq 1$.

%% file: parts/reduction.tex
\section{Reduction to Dueling Bandits}

\looseness -1 It is clear from the observation model \eqref{eq:feedback-intro} that any additive shift of the objective, i.e.~$\tilde f(x) = \fopt(x) + c$ for $c \in \RR$, can be absorbed in the unobserved confounding terms~$b_t$, hence rendering the observation sequences for $\fopt(x)$ and $\tilde f(x)$ indistinguishable. In particular, the learner can only hope to recover the true function up to an additive constant. Fortunately, to determine the best action $x^* = \argmax_{x \in \xX} \fopt(x)$, it suffices to estimate {\em reward differences} $\fopt(x^{1}) - \fopt(x^{2})$ for actions $x^1, x^2 \in \xX$, which is indeed possible. To do so, previous work in the linear setting relies on doubly-robust estimation \citep{krishnamurthy2018semiparametric,kim2019contextual}. Here  we take a different approach, and propose a generic reduction of the feedback model \eqref{eq:feedback-intro} to the \emph{dueling bandit} model \cite{yue2012k,sui2018advancements}. The reduction has the advantage that we can leverage existing algorithms for dueling bandits, which also eliminates the need to find low-variance sampling distributions for doubly-robust estimation. 

In dueling bandits, the learner chooses two actions $\xxt \in\xX$  and obtains (noisy) feedback on which of the two actions has higher reward. Meanwhile, the learner suffers regret for both actions, but the reward of each action is not observed. While most previous work on dueling bandits uses a binary feedback model, here, we are concerned with a quantitative version of the same setting, which is a special case of {\em linear partial monitoring} \cite{lin2014combinatorial}. Specifically, for a reward function $f : \xX \rightarrow \RR$ and actions $\xxt \in \xX$, we define \emph{quantitative dueling feedback} as follows:
\begin{align}
	\dd_t = \fopt(x_t^{1}) - \fopt(x_t^{2}) + \xi_t\,,\label{eq:dueling}
\end{align}
where $\xi_t$ is $\rho^2$-sub-Gaussian observation noise and $\rho$ is known to the learner. 
The distribution of $\xi_t$ is allowed to depend on $(x_t^{1}, x_t^{2})$ as long as the sub-Gaussian tail assumption is satisfied uniformly over all actions. 
Note that this includes binary feedback typically used for dueling bandits, and, more generally, bounded noise distributions.
Next, we present two reductions schemes to generate the dueling bandit feedback \eqref{eq:dueling} from confounded observations \eqref{eq:feedback-intro}. 

\paragraph{Two-Point Reduction} The first scheme uses two confounded observations to construct the dueling bandit feedback. Given two inputs $\xxt \in \xX$ in round $t$, we evaluate both points according to \eqref{eq:feedback-intro}, where the order of evaluation is uniformly randomized. The two observations are
\begin{align*}
	y_t^{1} &= \fopt(x_t^{1}) + b_{2t+i_t} + \epsilon_{2t+i_t}\,,\\
	y_t^{2} &= \fopt(x_t^{2}) + b_{2t + 1- i_t} + \epsilon_{2t+1 - i_t}\,,
\end{align*}
where $i_t \sim \Bernoulli(0.5)$. We then define 
\begin{align}
	\dd_t = y_t^{1} - y_t^{2}\,. \label{eq:two-point}
\end{align}
Assuming that $b_{2t}$ and $b_{2t+1}$ are fixed before either of $\xxt$ is chosen by the learner and using that the observation noise $\epsilon_t$ is zero-mean, one easily confirms that $\EE[\dd_t] = \fopt(x_t^{1})- \fopt(x_t^{2})$. We further use the following properties of sub-Gaussian random variables. Any random variable $X$ such that $X \in [-B,B]$ is $B^2$-sub-Gaussian. For two independent random variables $X_1, X_2$ that are $B_1^2$- and $B_2^2$-sub-Gaussian respectively, $X_1 + X_2$ is $(B_1^2 + B_2^2)$-sub-Gaussian. Hence if $|b_{2t} - b_{2t+1}| \leq D_{\max}$, it follows that \looseness=-1
\begin{align*}
	\xi_t &= \dd_t - \EE[\dd_t]\\
	&= b_{2t+i_t} + \epsilon_{2t+i_t} - (b_{2t + 1- i_t} + \epsilon_{2t+1 - i_t})
\end{align*} 
 is $(D_{\max}^2 + 2\sigma^2)$-sub-Gaussian.

\paragraph{One-Point Reduction} Perhaps surprisingly, one can also construct the dueling bandit feedback from a {\em single} observation using randomization.
For given inputs $\xxt \in \xX$ we choose one point uniformly
at random and evaluate the confounded function \eqref{eq:feedback-intro} to obtain a single observation
\begin{align*}
	y_t = f(x_t^{(1+i_t)}) + b_t + \epsilon_t\,,
\end{align*}
where $i_t \sim \Bernoulli(0.5)$. The dueling bandit feedback is
\begin{align}
	\dd_t =  (-1)^{i_t} \, 2 y_t \,. \label{eq:one-point}
\end{align}
Again, we get an unbiased observation of the reward difference, $\EE[\dd_t] = \fopt(x_t^{1}) - \fopt(x_t^{2})$. Further, if $|b_t| \leq C_{\max}$, then $\xi_t = \dd_t - \EE[\dd_t]$ is $4(C_{\max}^2 + \sigma^2)$-sub-Gaussian. Compared to the two-point reduction, here the sub-Gaussian variance $\rho$ depends on the absolute value $|b_t|$ of the confounding term instead of the difference $|b_t - b_{t+1}|$. On the other hand, the one-point sampling scheme only requires $b_t$ to be fixed before the choice of $x_t$, but may depend on \emph{all} previous actions and observations.

%% file: parts/ids.tex
\section{Information-Directed Sampling}\label{sec:ids}
With the reduction to dueling feedback, we are set to readily apply any dueling bandit algorithm in the confounded setting. Furthermore, dueling bandits (as defined in \eqref{eq:dueling}) are a special case of partial monitoring, for which also several algorithms exist. 
In the following, we adapt the {\em information-directed sampling} (IDS) approach \citep{Russo2014learning}, more specifically, the version proposed by \citet{kirschner20partialmonitoring} for linear partial monitoring. The main reason for this choice is that IDS works with \emph{quantitative dueling feedback}, whereas most other work on dueling bandits focuses on settings with Bernoulli likelihood. Also, IDS can be formulated as a \emph{kernelized algorithm}, and comes with \emph{theoretical guarantees} on the regret. However, a direct adaptation of IDS to the dueling setting as proposed by \citet{kirschner20partialmonitoring} requires $\oO(|\xX|^4)$ computational steps per round for finite action sets. In the following, we introduce an approximation of IDS, which obtains the same theoretical guarantees and only requires to optimize a simple score function over the action set. The resulting efficient kernelized dueling bandit algorithm may be of independent interest. On a high level, IDS samples actions from a distribution that minimizes a trade-off between an estimate of the regret and an information gain, as we elaborate below.

\paragraph{Kernel Regression for Dueling Feedback}
The first step is to set up \emph{kernelized least-squares regression} for the dueling bandit feedback. Recall that we assume that $\hH$ is a RKHS with kernel function $k: \xX \times \xX \rightarrow \RR$ and $\fopt \in \hH$ with $\|\fopt\|_\hH \leq B$.
In round $t$, the learner has already collected data $\dD_t = \{(\xx{x_s}, \dd_s)\}_{s=1}^{t-1}$, where $\xx{x_s}\in \xX$ is the input pair chosen at step $s$, and $\dd_s$ is quantitative dueling bandit feedback defined in Eq.~\eqref{eq:dueling}. The kernel least-squares estimator with regularizer $\lambda > 0$ is
\begin{align}
	\hat f_t = \argmin_{f \in \hH} \sum_{s=1}^{t-1} \left(f(x_s^{1}) - f(x_s^{2}) - \dd_s\right)^2 + \lambda \|f\|_{\hH}^2
\end{align}
The solution corresponds to the posterior mean of the Gaussian process model with kernel $k$ and prior variance $\lambda^{-1}$ and can be computed in closed form. Let $\mathbf \dd_t = [d_1, \dots, d_{t-1}]^\T$ be the vector which collects the observations and define $K_t \in \RR^{t-1\times t-1}$ and $k_t(x) \in \RR^{t-1}$  for $x \in \xX$ as follows:
\begin{align*}
	[K_t]_{ij} &\deq k(x_i^1,x_j^1) - k(x_i^1,x_j^2) - k(x_i^2,x_j^1) + k(x_i^2, x_j^2)\\
	[k_t(x)]_i &\deq k(x,x_i^1) - k(x,x_i^2)
\end{align*}
The least-squares solution $\hat f_t$ evaluated at any $x \in \xX$ is $\hat{f}_t(x) = k_t(x)^\T (K_t + \lambda \eye)^{-1} \mathbf \dd_t$, where $\eye$ is the identity matrix in the appropriate dimension. To compute uncertainty estimates, we further define 
\begin{align}
	k_t(x,y) &\deq k(x,y) - k_t(x)^\T(K_t + \lambda\eye)^{-1}k_t(y)\,, \nonumber\\
	\psi_t(x,z) &\deq k_t(x,x)^2 + k_t(z,z)^2 - 2 k_t(x,z)\,. \label{eq:psi-def}
\end{align} 
For any $t \geq 1$, $x_1,x_2 \in \xX$, the estimate $\hat f_t$ satisfies with probability at least $1-\delta$,
\begin{align}
	&\left|\hat f_t(x^1) - \hat f_t(x^2) - \big(\fopt(x^1) - \fopt(x^2)\big)\right|^2\nonumber\\
	& \qquad \qquad\qquad \qquad \qquad \qquad \leq \beta_{t,\delta} \psi_t(x^1, x^2)\,,\label{eq:confidence}
\end{align}
where the \emph{confidence coefficient} is chosen as follows,
\begin{align*}
	\beta_{t,\delta}^{1/2} &\deq \rho\sqrt{\log \det( \eye + \lambda^{-1} K_t)  + 2\log \tfrac{1}{\delta}} + \sqrt{\lambda}B\,.
\end{align*}
The confidence bound is by \citet[Theorem 3.11]{AbbasiYadkori2012}, which improves upon earlier results by \citet{Srinivas2009}.

\paragraph{Gap Estimate} We use $\hat f_t$ to define an estimate $\hat \Delta_t(x)$ of the suboptimality gaps $\Delta(x) = \fopt(x^*) - \fopt(x)$ for all $x \in \xX$.
Let $\hx_t = \argmax_{x \in \xX} \hat f_t(x)$ be the empirical estimate of the maximizer. Note that $\hx_t$ is always defined, despite the fact that in general we can determine $\hat f_t$ only up to a constant shift (which does not affect the maximizer). Define
\begin{align*}
	\delta_t &\deq \max_{z \in \xX} \hat f_t(z) - \hat f_t(\hx_t) + \beta_{t,\delta}^{1/2} \psi_t(\hx_t, z)^{1/2}\,,
\end{align*}
where $\psi_t(x,z)$ is defined in \eqref{eq:psi-def}. Intuitively, $\delta_t$ is the largest plausible regret that the learner can occur from playing $\hx_t$ given the confidence estimates. The \emph{gap estimate} is defined for any $x \in \xX$ as follows:
\begin{align}
	\hat \Delta_t(x) \deq \delta_t + \hat f_t(\hx_t) - \hat f_t(x) \,. \label{eq:gap}
\end{align}
The gap estimate satisfies an upper bound on the true gaps, provided that \eqref{eq:confidence} holds, as summarized in the next lemma.
\begin{lemma}\label{lem:gap-bound}
	With probability $1-\delta$, for all $x \in \xX$ and $t \geq 1$,
	\begin{align*}
		\Delta(x) \leq 2 \hat \Delta_t\big(x)\,.
	\end{align*}
\end{lemma}
\begin{proof}
	With probability at least $1-\delta$,
	\begin{align*}
		\Delta(x) &= \fopt(x^*) - \fopt(x)\\
		&= \max_{z \in \xX} f(z) - f(\hx_t) + f(\hx_t) - f(x)\\
		&\stackrel{(i)}{\leq} \max_{z \in \xX} \hat f_t(z) - \hat f_t(\hx_t) + \beta_{t,\delta}^{1/2} \psi_t(z, \hx_t)^{1/2}\\
		&\qquad \qquad + \hat f_t(\hx_t) - \hat f_t(x) + \beta_{t,\delta}^{1/2} \psi_t(\hx_t, x)^{1/2}\\
		&\stackrel{(ii)}{=}\hat \Delta_t(x) + \beta_{t,\delta}^{1/2}\psi_t(\hx_t, x)^{1/2}\,.
	\end{align*}
	Here, $(i)$ uses \eqref{eq:confidence} twice and $(ii)$ is by definition of $\hat \Delta_t(x)$.
	On the other hand, 
	\begin{align*}
		\hat \Delta_t(x) &=  \max_{z \in \xX} \hat f_t(z) - \hat f_t(x) + \beta_{t,\delta}^{1/2}\psi_t(z, \hx_t)^{1/2}\\
		&\geq\beta_{t,\delta}^{1/2}\psi_t(\hx_t, x)^{1/2}\,.
	\end{align*}
We used that $\psi_t(\hx_t, x) = \psi_t(x, \hx_t)$ in the last step. 
	The claim follows with the last two displays combined.
\end{proof}

Note that this gap estimate is different from the choice proposed by \citet{kirschner20partialmonitoring}, and importantly, allows us to reduce computational complexity.

\paragraph{Information Gain} IDS further requires an information gain function $I_t(\xxx)$ defined for each action, which in our case consists of an input pair, $\xxx \in \xX$. Several choices were proposed by \citet{kirschner20partialmonitoring}. Here we use the differential of the log-determinant potential, \looseness=-1
\begin{align}
	I_t(\xxx) &\deq \log(1 + \psi_t(x^1, x^2))\,,\label{eq:infogain}\\
	&= \log \frac{\det (\eye + K_{t+1})}{\det (\eye + K_{t})} \nonumber
\end{align}
This choice of information gain corresponds to the Bayesian mutual information $\II_t\big(\dd_t;\fopt|(\xx{x_t})=(\xxx)\big)$ in the Gaussian process model. In particular, the \emph{total information gain} for \eqref{eq:infogain} resembles the Gaussian entropy
\begin{align}
	\gamma_n=\sum_{t=1}^n I_t(x_t^1, x_t^2) = \log \det(\eye_n + K_{n+1})\,.
\end{align}
The log-determinant depends on the kernel function and upper bounds are known for many popular choices. For example, the linear kernel, $k(x,x') = \ip{x,x'}$, satisfies $\gamma_n \leq \oO(d \log(n))$ and the RBF kernel, $k(x,x') = \exp(-\|x-x'\|_2^2/2)$, satisfies $\gamma_n \leq \log(n)^{d+1}$ \citep[Theorem 5]{Srinivas2009}.

\LinesNumbered
\RestyleAlgo{ruled}
\begin{algorithm2e}[t]
	\DontPrintSemicolon
	\SetAlgoVlined
	\SetAlgoNoLine
	\SetAlgoNoEnd
	\SetKwInput{Input}{Input}
	\caption{Approx.~IDS for Dueling Feedback} \label{alg:ids}
	\Input{Action set $\xX$, confidence coefficient $\beta_{t,\delta}(\delta)$,\\\texttt{DuelingFeedback($\xxt$)}, e.g.~via (\ref{eq:two-point}) or (\ref{eq:one-point})}
	\For{$t=1,2,3, \dots$}{
		$\hx_t \gets \argmax_{x \in \xX} \hat f_t(x)$\;
		$\hat \Delta_t(x) \gets \delta_t + \hat f_t(\hx_t) - \hat f_t(x)$ \tcp*{Eq.\ \eqref{eq:gap}}
		$I_t(x) \gets \log(1 + \psi_t(\hx_t, x))$ \tcp*{Eq.\ \eqref{eq:infogain}}
		$x_t, p \gets \argmin_{x \in \xX,p \in [0,1]} \frac{\left((1-p)\delta_t + p \hat \Delta_t(x)\right)^2}{pI_t(x)}$\;
		$B_t\sim \Bernoulli(p)$\;
		\If{$B_t == 1$, }{
			$(\xxt) \gets (\hx_t, x_t)$\;
			$\dd_t \gets $\texttt{DuelingFeedback($\xxt$)}\;
		}\Else{
		$(\xxt) \gets (\hx_t, \hx_t)$ \tcp*{No feedback}
	}
	}
\end{algorithm2e}

\paragraph{(Approximate) Information-Directed Sampling}
Given the gap estimate $\hat \Delta_t(x)$ and information gain $I_t(x^1, x^2)$, we optimize the following trade-off jointly over $\xX \times [0,1]$,
\begin{align}
	z_t, p_t \deq \argmin_{z \in \xX,p \in [0,1]} \frac{\big((1-p)\delta_t + p \hat \Delta_t(z)\big)^2}{pI_t(\hx_t, z)}\,. \label{eq:ids-approx}
\end{align}
Note that for fixed $z$, the optimal trade-off probability is $p(z) = \min\left(\frac{\delta_t}{\hat \Delta_t(z) - \delta_t}, 1\right)$. Consequently, IDS samples the pair $(\hx_t, z_t)$ with probability $p_t$, and otherwise, with probability $1-p_t$, the greedy pair $(\hx_t, \hx_t)$. Note that choosing the same action for dueling feedback provides no information, and in particular the least-squares estimate remains unchanged. The approach is summarized in Algorithm~\ref{alg:ids}.

%

\paragraph{Computational Complexity}
\looseness -1 As usual, computing the kernel estimates requires to invert the kernel matrix $K_t$. With incremental updates, the exact kernel estimate and all related quantities can be computed in $\oO(d^2n^3)$ steps in total over $n$ rounds. Note that $\hx_t$ and $\delta_t$ can be computed in $\oO(|\xX|)$ steps assuming that all kernel quantities have been pre-computed. 
The trade-off \eqref{eq:ids-approx} can be computed in $\oO(|\xX|)$ steps, since it only requires to evaluate the gap estimates $\hat \Delta(x)$ and information gain $I_t(\hx_t, x)$ for all $x \in \xX$. Therefore the overall complexity  is $\oO(d^2n^3|\xX|)$. Of course, Algorithm \ref{alg:ids} can also be applied in the linear setting without kernelization, in which case the overall complexity is $\oO(d^2n|\xX|)$.

\looseness -1 We remark that \eqref{eq:ids-approx} is an approximation of the IDS trade-off proposed by \citet{kirschner20partialmonitoring}, which requires to optimize a similar quantity over distributions $\sP(\xX \times \xX)$. This is also possible, but the direct implementation suggested in \citep{kirschner20partialmonitoring} requires $\oO(|\xX|^4)$ compute steps per round to calculate the gap estimates and to find the IDS distribution.

\subsection{Regret Bounds}

In the language of linear partial monitoring, the dueling feedback \eqref{eq:dueling} is a so-called \emph{locally observable game}, which informally means that any reward difference $f(x) - f(x')$ can be estimated from playing actions which have no more regret than playing either $x$ or $x'$ alone  \citep[Appendix C.5]{kirschner20partialmonitoring}. For dueling bandits, IDS (without the approximation and sampling scheme that we introduce here) has regret at most $R_n \leq \oO(\sqrt{n \beta_n \gamma_n})$, see \citet[Corollary 18]{kirschner20partialmonitoring}. Here we show that Algorithm~\ref{alg:ids} satisfies a similar result. Note that the regret guarantee applies generally to settings with quantitative dueling feedback, where we define regret as follows:
\newcommand{\RD}{R^\text{\textit {duel}}}
\begin{align*}
	\RD_n = \sum_{t=1}^n \Delta(x_t^1) + \Delta(x_t^2)\,.
\end{align*}
\newcommand{\bound}{\sqrt{n \beta_{n,\delta} (\gamma_n + \log\tfrac{1}{\delta})}}
\begin{theorem}\label{thm:worst-case}
	For $\rho^2$-sub-Gaussian dueling bandit feedback \eqref{eq:dueling},  Algorithm \ref{alg:ids} satisfies with probability at least $1-\delta$,\looseness=-1
	\begin{align*}
		\RD_n \leq \oO\Big(\bound \Big)\,.
	\end{align*}
\end{theorem}
Note that the learner requires knowledge of the sub-Gaussian variance $\rho^2$ and the Hilbert norm bound $\|f\|_{\hH} \leq B$, which appear in the definition of the confidence coefficient $\beta_{n,\delta}$.
The regret guarantee has the same scaling as the best known bound for GP-UCB in standard Bayesian optimization \cite{Srinivas2009,Chowdhury2017}. Note that $\beta_{n,\delta} = \oO\left(\gamma_n + \log \frac{1}{\delta}\right)$, hence combined with bounds for the total information gain $\gamma_n$, we can derive bounds for specific choices of the kernel function. For example in the linear setting, we get $\RD_n \leq \oO\left(\rho \sqrt{n} (d\log(n) + \log \frac{1}{\delta})\right)$, which is the same as for LinUCB \citep{Abbasi2011improved}. For the RBF kernel, we get $\RD_n \leq \oO\left(\rho \sqrt{n} (\log(n)^{2d+2} + \log \frac{1}{\delta})\right)$.

Applied to the confounded setting with either the one- or two-point reduction, we get the following result. Note that depending on whether the learner requires one or two evaluations per round, the timescale differs by a factor of two.
\begin{corollary}\label{cor:confounded}
	In the confounded setting \eqref{eq:feedback-intro} with $\sigma^2$-sub-Gaussian observation noise and dueling feedback obtained via the one-point reduction \eqref{eq:one-point}, the regret of Algorithm \ref{alg:ids} satisfies with probability at least $1-\delta$,
	\begin{align*}
		R_n \leq \oO\Big((C_{\max}  + \sigma)\bound \Big)\,,
	\end{align*}
	assuming that $\max_{t\in[n]} b_t \leq C_{\max}$ and the adversary is allowed to choose $b_t$ depending on all previous actions and observations, $\{x_s,y_s\}_{s=1}^{t-1}$. 
	
	With the two-point reduction \eqref{eq:two-point}, Algorithm \ref{alg:ids} satisfies with probability at least $1-\delta$,
	\begin{align*}
		R_n \leq \oO\Big((D_{\max}  + \sigma)\bound \Big)\,,
	\end{align*}
	assuming that $\max_{t\in[n]} |b_{2t} - b_{2t +1}| \leq D_{\max}$ and the adversary is allowed to choose $b_t$ depending on all but the last two actions and observations, $\{x_s,y_s\}_{s=1}^{t-2}$. 
\end{corollary}
Note that the algorithm requires knowledge of the bound $C_{\max}$ or $D_{\max}$ respectively, which is needed to compute $\beta_{t,\delta}$. This is in line with the previous work in the linear setting \citep{krishnamurthy2018semiparametric,kim2019contextual}. Removing or weakening this assumption is an interesting direction for future work. Note that in the linear case with the one-point reduction method, our regret bound is $R_n \leq \tilde \oO(C_{\max} d \sqrt{n})$, which matches the result by \citet{krishnamurthy2018semiparametric}. The dependence on $d$ and $n$ cannot be improved even in the un-confounded linear bandit setting for general $\xX$ \citep[Theorem 24.1]{lattimore2018bandit}. The results for the two-point reduction requires a stronger assumption on the sequence $(b_t)_{t=1}^{2n}$, but the regret only depends on the differences $|b_{2t} - b_{2t+1}|$. Hence, the result assures sub-linear regret even for settings where the confounding terms are unbounded, for example when the objective function is subject to drift.

\begin{proof}[Proof of Theorem \ref{thm:worst-case}]
	First, by Lemma \ref{lem:gap-bound} with probability at last $1-\delta$, $\Delta(x_t) \leq 2 \hat \Delta_t(x_t)$. We extend the definition of the gap estimate to two points, $\hat \Delta_t(x^1, x^2) = \hat \Delta(x^1) + \hat \Delta(x^2)$. For a sampling distribution $\mu \in \sP(\xX \times \xX)$, denote the expected gap $\hat \Delta_t(\mu) \deq \EE_{x^1,x^2 \sim \mu}[\hat \Delta_t(x^1, x^2)]$ and the expected information gain $I_t(\mu) = \EE_{x^1,x^2 \sim \mu}[I_t(x^1,x^2)]$. The \emph{information ratio} is defined as follows:
	\begin{align}
		\Psi_t(\mu) \deq \frac{\hat \Delta_t(\mu)^2}{I_t(\mu)} \,.\label{eq:info-ratio}
	\end{align}
	Let $(\mu_t)_{t=1}^n$ be the sequence of sampling distributions $\mu_t \in \sP(\xX \times \xX)$ defined by Algorithm \ref{alg:ids}, 
	\[\mu_t = (1-p_t) e_{(\hx_t, \hx_t)} + p_t e_{(\hx_t, z_t)}\,,\]
	 where $e_x$ denotes a Dirac on $x \in \xX$.
	By \citep[Lemma 1,][]{kirschner20partialmonitoring}, with probability $1-\delta$,
	\begin{align*}
		\RD_n \leq \sqrt{\sum_{t=1}^n \Psi_t(\mu_t) \left( \gamma_n + \oO\Big(\log \frac{1}{\delta}\Big)\right)} + \oO\left(\log \frac{n}{\delta}\right)\,.
	\end{align*}
	The claim in the theorem follows if we show that the information ratio is bounded such that $\Psi_t(\mu_t) \leq \oO(\beta_{n, \delta})$. To this end, note that
		\begin{align}
		\Psi_t(\mu_t) &= \frac{\big((1-p_t) 2 \delta_t + p_t (\hat \Delta_t(z_t) + \delta_t)\big)^2}{pI_t(\hx_t, z_t)} \nonumber\\
		 &\leq  \min_{x \in \xX, p \in [0,1]} \frac{4 \big((1-p)\delta_t + p \hat \Delta_t(x)\big)^2}{p I_t(\hx_t, x)} \nonumber\\
		 &\leq  \min_{x \in \xX} \frac{4 \hat \Delta_t(x)^2}{I_t(\hx_t, x)}  \label{eq:proof-1}\,,
	\end{align}
	where the first inequality follows from $\delta_t \leq \hat \Delta_t(x)$ and the definition of $z_t$ and $p_t$, and the second inequality sets $p=1$. 
	
	\looseness=-1 On the other hand, using that the kernel is bounded, $k(x,x) \leq 1$, one easily checks that $\psi_t(x^1, x^2) \leq 4$. With $a \leq 3 \log(1 + a)$ for all $a \in [0,4]$ we find for $\xxx \in \xX$
	\begin{align}
		\psi_t(\xxx) \leq 3 \log(1 + \psi_t(\xxx)) = 3I_t(\xxx)\,. \label{eq:proof-2}
	\end{align}
	Next, define $\tilde z_t = \argmax_{x \in \xX} \hat f_t(x) + \beta_{t,\delta}^{1/2} \psi_t(\hat x_t, x)^{1/2}$ and observe that $\hat \Delta_t(\tilde z_t) = \beta_{t,\delta}^{1/2} \psi_t(\hat x_t, \tilde z_t)^{1/2}$. The claim follows with \eqref{eq:proof-1} from noting that
	\begin{align}
		\Psi_t(\mu_t) &\leq \frac{4 \hat \Delta_t(\tilde z_t)^2}{I_t(\tilde z_t)} \leq 12 \frac{\beta_{t,\delta} \psi_t(\hat x_t, \tilde z_t)}{ \psi_t(\hat x_t, \tilde z_t)} = 12 \beta_{t,\delta}\,.\label{eq:proof-3}
	\end{align}
	The result follows from the fact that $\beta_{t,\delta}$ is monotonically increasing.
\end{proof}

Algorithm \ref{alg:ids} also satisfies a gap-dependent bound for finite action sets $\xX$. Let $\Delta_{\min} = \min_{x \neq x^*} \Delta(x)$ be the smallest non-zero gap and assume that $x^*$ is unique. The theorem applies to the confounded setting via the reduction method similar to Corollary \ref{cor:confounded}.
\begin{theorem}\label{thm:logarithmic}
	Assuming that $x^*$ is unique, the regret of Algorithm \ref{alg:ids} satisfies with probability at least $1-\delta$,
	\begin{align*}
		\RD_n \leq \oO\left(\Delta_{\min}^{-1} \beta_{n,\delta}( \gamma_n + \log \tfrac{n}{\delta})\right)\,.
	\end{align*}
\end{theorem}
For linear bandits, the regret bound reads $\RD_n \leq \oO(\Delta_{\min}^{-1}d^2 \log(n)^2)$ and for the RKHS setting with RBF kernel, the bound is $\RD_n \leq \oO(\Delta_{\min}^{-1} \log(n)^{2d+2})$.

\begin{proof}[Proof of Theorem \ref{thm:logarithmic}] Our proof uses the strategy introduced by \citet{kirschner2020asymptotically} that relies on finding an instance-dependent bound on the information ratio.
		Let $\Psi_t(\mu_t)$ be the information ratio defined in \eqref{eq:info-ratio}. We apply \citep[Lemma 1 with $\eta=0$]{kirschner20partialmonitoring} to the scaled information gain $\tilde I_t = \beta_{t,\delta} I_t$, to find with probability $1-\delta$,
	\begin{align*}
		\sum_{t=1}^n \hat \Delta_t(x_t^1, x_t^2)\leq \sqrt{\sum_{t=1}^n \frac{\Psi_t(\mu_t)}{\beta_{t,\delta}} \left(\beta_{n,\delta}( \gamma_n + \oO(\log \tfrac{n}{\delta}))\right)}
	\end{align*}
	We condition now on the event that the previous equation and the confidence estimate in \eqref{eq:confidence} hold simultaneously. As before, let $\tilde z_t = \argmax_{x \in \xX} \hat f_t(x) + \beta_{t,\delta}^{1/2} \psi_t(\hat x_t, x)^{1/2}$. We may assume that $\tilde z_t \neq \hat x_t$, since otherwise $\delta_t = 0$ and therefore $\Psi_t(\mu_t) = 0$. On the other hand, this implies that $2 \hat \Delta_t(\tilde z_t) \geq \Delta_{\min}$ by \eqref{lem:gap-bound} and using that $x^*$ is unique.
	
	Next, we reuse the inequality leading to \eqref{eq:proof-1} to find
	\begin{align*}
		\Psi_t(\mu_t) \leq \min_{p \in [0,1]} \frac{4 \big((1-p)\delta_t + p \hat \Delta_t(\tilde z_t)\big)^2 }{p I_t(\hat x_t, \tilde z_t)}\,. 
	\end{align*}
	First, consider the case where $2 \delta \leq \hat \Delta_t(\tilde z_t)$. Computing the minimizer of the previous display and using \eqref{eq:proof-2}, we find
	\begin{align*}
		\Psi_t(\mu_t) \leq  \frac{16 \delta_t \hat \Delta_t(\tilde z_t) }{I_t(\hat x_t, \tilde z_t)} \leq \frac{48 \delta_t \beta_{t, \delta}}{\hat \Delta_t(\tilde z_t)} \leq \frac{96 \delta_t \beta_{t, \delta}}{\Delta_{\min}}\,.
	\end{align*}
	For the other case where $2\delta_t > \hat \Delta_t(\tilde z_t)$, using \eqref{eq:proof-3} directly gives
	\begin{align*}
		\Psi_t(\mu_t) \leq 12 \beta_{t,\delta} \leq \frac{24 \delta_t \beta_{t, \delta}}{\hat \Delta_t(\tilde z_t)} \leq \frac{48 \delta_t \beta_{t, \delta}}{\Delta_{\min}}\,.
	\end{align*}
	
	Finally, note that $\delta_t \leq \frac{1}{2}\hat \Delta_t(\xxt)$. Using the bound on the information ratio, and solving for the regret, we find
	\begin{align*}
		\sum_{t=1}^n \hat \Delta_t(x_t^1, x_t^2)  \leq \oO\left(\Delta_{\min}^{-1} \beta_{n,\delta}( \gamma_n + \log \tfrac{n}{\delta})\right)\,.
	\end{align*}
The claim follows with Lemma \ref{lem:gap-bound} by noting that,
\begin{align*}
	\RD_n &\leq \sum_{t=1}^n 2 \hat \Delta_t(x_t^1, x_t^2)\,.   \qedhere
\end{align*}
\end{proof}

\subsection{A Connection to Doubly-Robust Estimation}
In the finite-dimensional, linear case, the objective function is $\fopt(x) = \ip{x, \theta}$ for a fixed parameter $\theta \in \RR^d$. To obtain an estimate $\hat \theta_t$ of the unknown parameter $\theta$ directly from confounded data $\{(x_s, y_s=\ip{x_s, \theta} + b_s + \epsilon_s)\}_{s=1}^{t-1}$, \citet{krishnamurthy2018semiparametric} use a randomized policy $x_t \sim \mu_t$ and a doubly-robust estimation approach. For centered feature vectors $\bar x_t = \EE_{x \sim \mu_t}[x]$ and regularizer $\lambda > 0$, they define \looseness=-1
\begin{align}
	\Gamma_t &\deq \sum_{s=1}^{t-1} (x_t - \bar x_t)(x_t - \bar x_t)^\T + \lambda \eye_d\,,\nonumber\\
	\hat \theta_t &\deq \Gamma_t^{-1}\sum_{s=1}^{t-1} (x_t - \bar x_t)y_t.\label{eq:dr-estimator}
\end{align}
They further derive the following high-probability bound for doubly-robust estimator:
\begin{align}
	\|\hat \theta_t - \theta\|_{\Gamma_t}^2 \leq \oO(d \log(n) + \log(n/\delta) + \lambda)\,.\label{eq:dr-concentration}
\end{align}
Interestingly, when $\mu_t = \Uniform(\{x^1,x^2\})$ is chosen to randomize between two actions $x^1,x^2 \in \xX$, then this estimator coincides with the least-squares estimator that we obtain for the dueling bandit feedback.  This follows immediately from noting that $2(x_t - \bar x_t) = x^1 - x^2$. Further, the concentration bounds are on the same quantity, but the reduction avoids the detour to prove the (more general) concentration bound \eqref{eq:dr-concentration} and leads tighter bounds. We remark that in general, the BOSE algorithm by \citet{krishnamurthy2018semiparametric} requires to compute a sampling distributions $\mu_t$ supported on $d+1$ points.
We also note that \citet{kim2019contextual} propose another variant of the doubly-robust estimator \eqref{eq:dr-estimator}. This estimator coincides with our estimation scheme in the same way.

%% file: parts/experiments.tex
\begin{figure}[t]
	\includegraphics{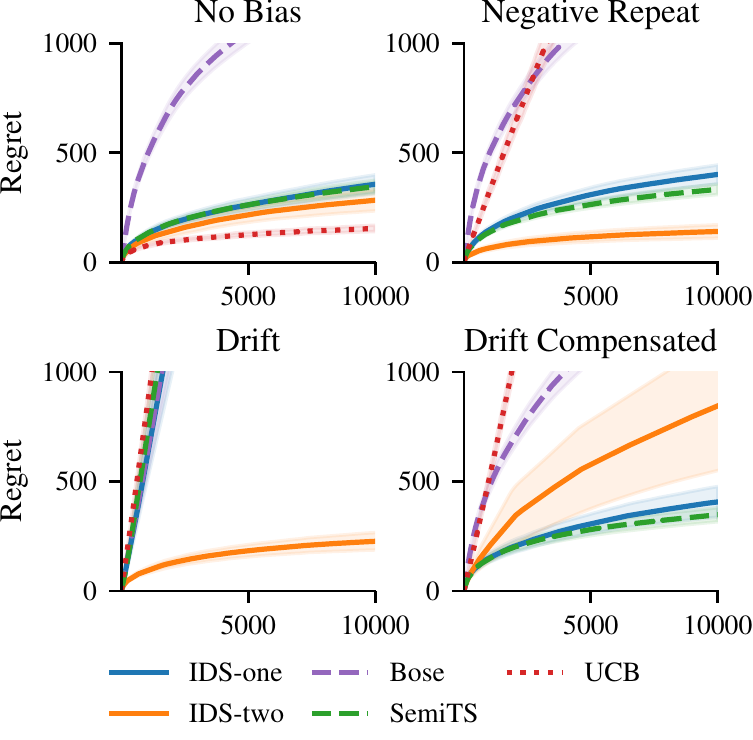}
	\vspace{-5pt}
	\caption{Benchmarks with randomly sampled action sets of size 20 and $d=4$. The confidence region shows 2$\times$ standard error over 50 repetition. Without confounding, UCB performs best, and is closely followed by IDS and Semi-TS. With additive bias the performance of UCB degrades significantly, whereas the robust methods maintain sublinear regret. Note that in the drift experiment (bottom left), the bias is unbounded, which leads to linear regret for all methods expect IDS-two. With bounded bias (right column), IDS-one and Semi-TS show similar performance, whereas the performance of IDS-two is varying. The BOSE algorithm shows sublinear behavior but is much more conservative. }
	\label{fig:linear}
\end{figure}

\begin{figure}[t]
	\includegraphics{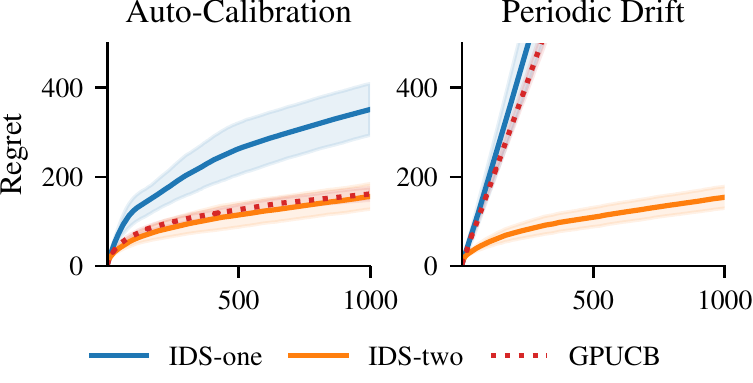}
	\vspace{-5pt}
	\caption{Performance of kernelized algorithms on the camelback benchmark. The confidence region shows 2$\times$ standard error over 50 repetitions. With bounded confounding that simulates a re-calibration process (left), GP-UCB is surprisingly competitive with IDS-two, followed by IDS-one. When the feedback is confounded by an unbounded, periodic drift (right), we observe linear regret of GP-UCB and IDS-one (because the bias is unbounded), whereas IDS-two maintains good performance despite the confounding.}
	\label{fig:camelback}
\end{figure}

\section{Experiments}
We evaluate the proposed method with the one-point reduction (IDS-one) and the two-point reduction (IDS-two) in two numerical experiments with confounded observations. To allow a fair comparison with the two-sample scheme, we account for the regret of both evaluations and scale the x-axis appropriately.

\subsection{Baselines}
\paragraph{UCB} For the linear setting, LinUCB \cite{auer2002confidencebounds} is implemented as in \citep[Figure 1]{Abbasi2011improved} using a regularizer $\lambda = 1$ and confidence coefficient $\beta_{t,\delta}^{1/2} = \sqrt{\log \det V_t + 2 \log \frac{1}{\delta}} + 1$. In the kernelized setting, we use GPUCB \cite{Srinivas2009} with an empirically tuned confidence coefficient $\beta_{n, \delta} = 1$. As shown by \citet{bogunovic2020corruption}, increasing the confidence coefficient to a larger value as required in the stochastic setting can lead to better robustness, although we did not see an improvement of performance in our experiments. 

\paragraph{BOSE} The BOSE algorithm \citep[Algorithm 1]{krishnamurthy2018semiparametric} uses the doubly-robust least-squares estimator \eqref{eq:dr-estimator}. We set the required concentration coefficient $\beta_{t,\delta}^{\text{DR}}$ to
\begin{align*}
	\beta_{t,\delta}^{\text{DR}} = \sqrt{d \log(1 + t/d) + 2\log\left(\tfrac{t}{\delta}\right)} + 1\,,
\end{align*}
where we drop (conservative) constants required for the theoretical results in favor of better empirical performance. BOSE requires to solve a convex-quadratic feasibility problem on the space of sampling distributions over the remaining plausible actions, and no specific computation method was recommended by the authors. We compute the sampling distribution by solving the saddle point problem stated in \citep[Appendix D]{krishnamurthy2018semiparametric} using exponentiated gradient descent.

\paragraph{SemiTS} The semi-parametric Thompson sampling is implemented as in \citep[Algorithm 1]{kim2019contextual}, with a less conservative over-sampling parameter $v=\sqrt{2 \log(t/\delta)}$. Our choice improves performance over the theoretical value. We also remark that SemiTS requires to compute the probability of each action being optimal under a Gaussian perturbation of the mean parameter. We do so by computing the empirical sampling probabilities from 1000 random samples per round, the alternative being to compute Gaussian integrals over $d$-dimensional polytopes $\cC_x = \{ \nu \in \RR^d : \ip{x, \nu} \geq \max_{x' \in \xX} \ip{x', \nu}  \}$. While SemiTS is significantly faster than BOSE in our implementation, computing the posterior probabilities accurately for larger action sets remains challenging.

In all experiments we set confidence level $\delta=0.05$.

\subsection{Environments}

\paragraph{Linear Reward} In the first experiment, we use a linear reward function $f(x) = \ip{x, \theta}$. For each repetition we sample $k=20$ actions uniformly on the $d=4$ dimensional unit sphere. We add Gaussian observation noise with variance $\sigma^2=1$, that is  $\epsilon_t \sim \nN(0, 1)$ in \eqref{eq:feedback-intro}. In this setting we compare to BOSE, SemiTS and  LinUCB \citep{auer2002confidencebounds,Abbasi2011improved}, where the latter does not directly deal with the confounding. We consider four different types of confounding: \textit{a)} \textit{no bias}; \textit{b)} the adversary repeats the last observation with a minus sign, $b_t = -y_{t-1}$, which makes it much harder to identify the best action \textit{(negative repeat}); \textit{c)} a continues \emph{drift}, $b_t = -0.1t$, i.e.\ unbounded confounding; and \textit{d)} same as the previous,  but with \emph{compensated drift}, $b_t = -0.1t + y_{t-1}$, thereby making the bias terms bounded but dependent on the previous observation.

The result is shown in Figure \ref{fig:linear}. As expected, in the unconfounded setting UCB works best, followed by both IDS variants and SemiTS with reasonable performance. With confounding, the regret of UCB is increased by a lot, whereas BOSE shows sublinear behaviour but is relatively inefficient. In the example with unbounded bias (drift) only IDS-two performs well, as in fact the theoretical assumptions for all other methods are invalidated. With bounded bias (i.e.~negative repeat and compensated drift), SemiTS and IDS-one are competitive, while IDS-two clearly outperforms the baselines in the \emph{negative repeat} experiment.

\paragraph{Camelback} Our second experiment is in the non-linear, kernelized setting with observation noise variance $\sigma^2=0.1$. As benchmark we choose the camelback function on the domain $[-2,2] \times [-1,1]$,
 \begin{align*}
	&f(x_1, x_2) = \\
	&- \min\left(x_1^2 \big(4 - 2.1x_1^2 + \tfrac{x_1^4}{3.}\big) + x_1x_2 + x_2^2(4x_2^2-4), 2.5\right)
\end{align*} 
\looseness -1 We discretize the input space using 30 points per dimension. The only direct competitor that we are aware of is the method of \citet{bogunovic2020corruption}. This method is, however, equivalent to GP-UCB \citep{Srinivas2009} with an up-scaled confidence coefficient. This suggests that the UCB approach is inherently robust up to a certain degree of corruption, which is also visible in our experiment. For both algorithms, we use an RBF kernel with lengthscale $0.2$ and regularizer $\lambda = 1$, and set $\beta_{n,\delta}=1$ in favor of better empirical performance.  We use two types of confounding that we expect is relevant in applications: \textit{a)} a calibration process, which monitors a moving average over the last 10 observations and adjusts the output range to $[-0.1, 0.1]$ whenever the average is no longer in this range; and \textit{b)} periodic drift of the objective, $b_t = \text{sin}(0.2t) - 0.1t$. Results are shown in Figure \ref{fig:camelback}. In the first variant GPUCB works surprisingly well despite the confounding and is on-par with IDS-two. With unbounded drift, both GPUCB and IDS-one obtain linear regret, whereas the performance of IDS-two in unaffected.

%% file: parts/conclusion.tex
\section{Conclusion}

We introduced randomized evaluation schemes based on pair-wise comparisons that make dueling bandit algorithms applicable to robust optimization with additive confounding. Moreover, we derived a kernelized dueling bandit algorithm based on recent ideas by \citet{kirschner20partialmonitoring}. The resulting algorithm satisfies worst-case and gap-dependent regret bounds on the cumulative regret and could be of broader interest in the dueling bandit setting. Our numerical experiments validate the theoretical findings.